\documentclass{article}

\usepackage{times}
\usepackage{graphicx} 
\usepackage{subfigure} 

\usepackage{natbib}

\usepackage{algorithm}
\usepackage{algorithmic}

\usepackage{hyperref}

\usepackage[accepted]{icml2014}

\usepackage{amsmath}
\usepackage{amssymb}

\newtheorem{theorem}{Theorem}

\newtheorem{proposition}{Proposition}
\newtheorem{definition}{Definition}

\def\n{{\mathbf n}}
\def\y{{\mathbf y}}

\newcommand{\C}{\mathcal{C}}
\renewcommand{\S}{\mathcal{S}}
\newcommand{\D}{\mathcal{D}}

\newcommand{\indep}{\perp\!\!\!\perp}

\newcommand{\eat}[1]{}

\renewcommand{\b}{\mathbf}
\newcommand{\I}{\mathbb{I}}
\newcommand{\E}{\mathbb{E}}
\newcommand{\thetab}{\boldsymbol{\theta}}
\newcommand{\mub}{\boldsymbol{\mu}}
\newcommand{\x}{\mathbf{x}}
\newcommand{\X}{\mathcal{X}}

\newcommand{\<}{\langle}
\renewcommand{\>}{\rangle}

\newcommand{\cov}{\mathrm{cov}}

\newcommand{\diag}{\mathrm{diag}}

\newcommand{\law}{\xrightarrow[]{D}}
\newcommand{\tilden}{\tilde{\mathbf{n}}}
\newcommand{\tI}{\tilde{\mathbf{I}}}

\newcommand{\z}{\mathbf{z}}
\newcommand{\tz}{\tilde{\mathbf{z}}}

\newcommand{\barn}{\z}

\DeclareMathOperator*{\GCGM}{GCGM}
\DeclareMathOperator*{\CGM}{CGM}

\icmltitlerunning{Gaussian Approximation of Collective Graphical Models}

\begin{document} 

\twocolumn[
\icmltitle{Gaussian Approximation of Collective Graphical Models}

\icmlauthor{Li-Ping Liu\texorpdfstring{$^1$}{1}}{liuli@eecs.oregonstate.edu}
\icmlauthor{Daniel Sheldon\texorpdfstring{$^2$}{2}}{sheldon@cs.umass.edu}
\icmlauthor{Thomas G. Dietterich\texorpdfstring{$^1$}{1}}{tgd@eecs.oregonstate.edu}
\icmladdress{\texorpdfstring{$^1$}{1}School of EECS, Oregon State University, Corvallis, OR
  97331 USA \\\texorpdfstring{$^2$}{2}University of Massachusetts,
  Amherst, MA 01002 and Mount Holyoke College, South Hadley, MA 01075}

\icmlkeywords{collective graphical model, Gaussian approximation, inference}

\vskip 0.3in
]

\begin{abstract} 
The Collective Graphical Model (CGM) models a population of
independent and identically distributed individuals when only
collective statistics (i.e., counts of individuals) are
observed. Exact inference in CGMs is intractable, and previous work
has explored Markov Chain Monte Carlo (MCMC) and MAP approximations
for learning and inference. This paper studies Gaussian approximations
to the CGM. As the population grows large, we show that the CGM
distribution converges to a multivariate Gaussian distribution (GCGM)
that maintains the conditional independence properties of the original
CGM.  If the observations are exact marginals of the CGM or marginals
that are corrupted by Gaussian noise, inference in the GCGM
approximation can be computed efficiently in closed form. If the
observations follow a different noise model (e.g., Poisson), then
expectation propagation provides efficient and accurate approximate
inference. The accuracy and speed of GCGM inference is compared to the
MCMC and MAP methods on a simulated bird migration problem. The GCGM
matches or exceeds the accuracy of the MAP method while being significantly
faster.
\end{abstract} 

\section{Introduction}

Consider a setting in which we wish to model the behavior of a
population of independent and identically distributed (i.i.d.)
individuals but where we can only observe collective count data. For
example, we might wish to model the relationship between education,
sex, housing, and income from census data. For privacy reasons, the
Census Bureau only releases count data such as the number of people
having a given level of education or the number of men living in a
particular region.  Another example concerns modeling the behavior of
animals from counts of (anonymous) individuals observed at various
locations and times.  This arises in modeling the migration of fish
and birds.

The CGM is constructed by first defining the {\it individual
  model}---a graphical model describing a single individual.  Let $\C$
and $\S$ be the clique set and the separator set of a junction tree
constructed from the individual model.  Then, we define $N$ copies of
this individual model to create a population of $N$
i.i.d. individuals. This permits us to define count variables $\n_A$,
where $\n_A(i_A)$ is the number of individuals for which clique $A \in
\C \cup \S$ is in configuration $i_A$. The counts $\n = (\n_A: A \in
\C \cup \S)$ are the sufficient statistics of the individual
model. After marginalizing away the individuals, the CGM provides a
model for the joint distribution of $\n$.

In typical applications of CGMs, we make noisy observations $\y$ that
depends on some of the $\n$ variables, and we seek to answer queries
about the distribution of some or all of the $\n$ conditioned on these
observations. Let $\y = (\y_D : D \in \D )$, where $\D$ is a set of
cliques from the individual graphical model and $\y_D$ contains counts
of settings of clique $D$. We require each $D \subseteq A$ for some
clique $A \in \C \cup \S$ the individual model.  In addition to the
usual role in graphical models, the inference of the distribution of
$\n$ also serves to estimate the parameters of the individual model
(e.g. E step in EM learning), because $\n$ are sufficient statistics
of the individual model.  Inference for CGMs is much more difficult
than for the individual model. Unlike the individual model, many
conditional distributions in the CGM do not have a closed form. The
space of possible configurations of the CGM is very large, because
each count variable $\n_i$ can take values in $\{0,\ldots,N\}$.

The original CGM paper, \citet*{Sheldon2011} introduced a Gibbs
sampling algorithm for sampling from $P(\n|\y)$.  Subsequent
experiments showed that this exhibits slow mixing times, which
motivated \citet*{Sheldon2013} to introduce an efficient algorithm for
computing a MAP approximation based on minimizing a tractable convex
approximation of the CGM distribution. Although the MAP approximation
still scales exponentially in the domain size $L$ of the
individual-model variables, it was fast enough to permit fitting CGMs
via EM on modest-sized instances ($L=49$).  However, given that we
wish to apply this to problems where $L=1000$, we need a method that
is even more efficient.

This paper introduces a Gaussian approximation to the CGM.  Because the
count variables $\n_C$ have a multinomial distribution, it is
reasonable to apply the Gaussian approximation. However, this approach
raises three questions. First, is the Gaussian approximation
asymptotically correct? Second, can it maintain the sparse dependency
structure of the CGM distribution, which is critical to efficient
inference? Third, how well does it work with natural (non-Gaussian)
observation distributions for counts, such as the Poisson
distribution?  This paper answers these questions by proving an
asymptotically correct Gaussian approximation for CGMs. It shows that
this approximation, when done correctly, is able to preserve the
dependency structure of the CGM. And it demonstrates that by applying
expectation propagation (EP), non-Gaussian observation distributions
can be handled. The result is a CGM inference procedure that gives
good accuracy and achieves significant speedups over previous methods.

Beyond CGMs, our main result highlights a remarkable property of
discrete graphical models: the asymptotic distribution of the
vector of sufficient statistics is a Gaussian graphical model with the
same conditional independence properties as the original model.

\eat{ 
This paper is organized as follows. Section 2 introduces the notation
and a precise statement of the CGM distribution. Section 3 presents
the Gaussian approximation and proves its structural and consistency
properties.  It shows that if the observations are exact (i.e., direct
observations of $\n_i$), then inference can be computed in closed
form. Section 4 discusses the issues that arise when the observations
are noisy and shows that an application of expectation propagation (EP)
allows efficient approximate inference.  The accuracy and speed of the
approximations are evaluated in Section 5 on a 
bird migration simulation.
}

\section{Problem Statement and Notation}


Consider a graphical model defined on the graph $G = (V,E)$ with $n$
nodes and clique set $\C$. Denote the random variables by
$X_1, \ldots, X_n$. Assume for simplicity all variables take values
in the same domain $\X$ of size $L$.
Let $\x \in \X^n$ be a particular configuration of the variables, and let
$\x_C$ be the subvector of variables belonging to $C$. For each clique
$C \in \C$, let $\phi_C(\b{x}_C)$ be a non-negative potential
function. Then the probability model is:
\begin{align}
\label{model1}
p(\x)
&= \frac{1}{Z} \prod_{C \in \C} \phi_{C}(\x_C) \nonumber\\
&= \exp \! \Big( \sum_{C \in \C} \sum_{i_C \in \X^{|C|}}\!
                      \theta_{C}(i_C) \!\cdot\! \I(\x_C = i_C) \!-\! Q(\thetab) \Big).
\end{align}
The second line shows the model in exponential-family form \citet{wainwright2008graphical}, where
$\I(\pi)$ is an indicator variable for the event or expression 
$\pi$, and $\theta_{C}(i_C) = \log \phi_{C}(i_C)$ is an entry of the vector
of natural parameters. The function $Q(\thetab) = \log Z$ is the
log-partition function.
Given a fixed set of parameters $\thetab$ and any subset $A \subseteq V$,
the \emph{marginal distribution} $\mub_A$ is the vector with entries
$\mu_A(i_A) = \Pr(X_A = i_A)$ for all possible $i_A \in \X^{|A|}$. In
particular, we will be interested in the clique marginals $\mub_C$ and
the node marginals $\mub_i := \mub_{\{i\}}$.

\textbf{Junction Trees.} Our development relies on the existence of a
\emph{junction tree} \cite{lauritzen1996graphical} on the cliques of
$\C$ to write the relevant CGM and GCGM distributions in closed
form. Henceforth, \emph{we assume that such a junction tree
  exists}. In practice, this means that one may need to add fill-in
edges to the original model to obtain the \emph{triangulated} graph
$G$, of which $\C$ is the set of maximal cliques. This is a clear
limitation for graphs with high tree-width. Our methods apply directly
to trees and are most practical for low tree-width graphs.  Since we
use few properties of the junction tree directly, we review only the
essential details here and review the reader to
\citet{lauritzen1996graphical} for further details.  Let $C$ and $C'$
be two cliques that are adjacent in $\mathcal{T}$; their intersection
$S = C \cap C'$ is called a \emph{separator}. Let $\S$ be the set of
all separators of $\mathcal{T}$, and let $\nu(S)$ be the number of
times $S$ appears as a separator, i.e., the number of different edges
$(C,C')$ in $\mathcal{T}$ for which $S = C \cap C'$.

\paragraph{The CGM Distribution.}
Fix a sample size $N$ and let $\x^{1}, \ldots, \x^N$ be $N$
i.i.d. random vectors distributed according to the graphical model
$G$. For any set $A \subseteq V$ and particular setting $i_A \in
\X^{|A|}$, define the count
\begin{align}
\n_A(i_A) &= \sum_{m=1}^N \I(\x_A^m = i_A).
\end{align}
Let $\n_A = (\n_A(i_A): i_A \in \X^{|A|})$ be the complete vector of
counts for all possible settings of the variables in $A$.
In particular, let $\n_u := \n_{\{u\}}$ be the vector of node counts. 
Also, let $\n = (\n_{A} : A \in \C \cup \S)$ be the combined vector of
all clique and separator counts---these are sufficient statistics
of the sample of size $N$ from the graphical model.
The distribution
over this vector is the CGM distribution. 
\begin{proposition}
\label{cgm-distribution}
Let $\n$ be the vector of (clique and separator) sufficient statistics
of a sample of size $N$
from the discrete graphical model \eqref{model1}.
The probability mass function of $\n$ is given by $p(\n; \thetab) = h(\n) f(\n; \thetab)$ where
\begin{align}
\label{cgm-prob}
f(\n; \thetab) = 
\exp \big( \!\!\!\!\!\!\! \sum_{C \in \C, i_C \in \X^{|C|}}
\!\!\!\!\!\!\! \theta_{C}(i_C) \cdot \n_{C}(i_C) - N Q(\thetab) \big)
\end{align}
\begin{multline}
\label{cgm-base-measure}
h(\n) = N! \cdot 
\frac{\prod_{S \in \S}\prod_{i_S \in \X^{|S|}} \Big(\n_S(i_S)!\Big)^{\nu(S)}}
{ \prod_{C \in \C} \prod_{i_C \in \X^{|C|}} \n_{C}(i_C)!}  \\
\prod_{S \sim C \in \mathcal{T}, i_S \in \X^{|S|}} \!\!\!\!\!
\I \Big(\n_S(i_S) = \sum_{i_{C\setminus S}} \n_{C}(i_S, i_{C \setminus S}) \Big) \; \cdot \; \\
\prod_{C \in \C} \I \big( \sum_{i_C \in \X^{|C|}} \n_C(i_C) = N \big).
\end{multline}
Denote this distribution by $\CGM(N, \thetab)$. 
\end{proposition}
Here, the notation $S \sim C \in \mathcal{T}$ means that $S$ is
adjacent to $C$ in $\mathcal{T}$.
This proposition was first proved in nearly this form by
\citet{sundberg1975some} (see also
\citet{lauritzen1996graphical}). Proposition~\ref{cgm-distribution}
differs from those presentations by writing $f(\n; \thetab)$ in terms
of the original parameters $\thetab$ instead of the clique and
separator marginals $\{\mub_C, \mub_{S}\}$, and by including hard constraints in the base
measure $h(\n)$. The hard constraints enforce 
consistency of the sufficient statistics of all cliques on their
adjacent separators, and were treated
implicitly prior to \citet{Sheldon2011}. 
A proof of the equivalence between our expression for $f(\n; \thetab)$
and the expressions from prior work is given in the supplementary material. 
\citet{dawid1993hyper} refer to the same distribution as the
\emph{hyper-multinomial} distribution due to the fact that it follows
conditional independence properties analogous to those in the original
graphical model.
\begin{proposition}
\label{cgm-independence}
Let $A, B \in \S \cup \C$ be two sets that are separated by the
separator $S$ in $\mathcal{T}$. Then $\n_{A} \indep \n_{B} \mid
\n_S$.
\end{proposition}
\begin{proof}
The probability model $p(\n; \thetab)$ factors over the clique and
separator count vectors $\n_C$ and $\n_S$. The only factors where two
different count vectors appear together are the consistency
constraints where $\n_S$ and $\n_C$ appear together if $S$ is adjacent
to $C$ in $\mathcal{T}$. Thus, the CGM is a graphical model with the
same structure as $\mathcal{T}$, from which the claim follows.
\end{proof}

\section{Approximating CGM by the Normal Distribution} \label{section_gcgm}

In this section, we will develop a Gaussian approximation, GCGM, of
the CGM and show that it is the asymptotically correct distribution as
$M$ goes to infinity. We then show that the GCGM has the same
conditional independence structure as the CGM, and we explicitly derive
the conditional distributions. These allow us to use Gaussian message
passing in the GCGM as a practical approximate inference method for
CGMs.

We will follow the most natural approach of approximating the CGM
distribution by a multivariate Gaussian with the same mean and
covariance matrix. The moments of the CGM distribution follow directly
from those of the indicator variables of the individual model: Fix an
outcome $\x = (x_1, \ldots, x_n)$ from the individual model and for
any set $A \subseteq V$ let $\mathbf{I}_A = \big(\I(\x_A = i_A) : i_A
\in \X^{|A|}\big)$ be the vector of all indicator variables for that
set. The mean and covariance of any such vectors are given by
\begin{align}
\E[\b{I}_A] &= \mub_A \\
\label{covariance}
\cov(\b{I}_A, \b{I}_B) &= \< \mub_{A,B} \> - \mub_A \mub_B^T. 
\end{align}
Here, the notation $\< \mub_{A,B} \>$ refers to the matrix whose $(i_A,
i_B)$ entry is the marginal probability $\Pr(X_A = i_A, X_B =
i_B)$. Note that Eq.~\eqref{covariance} follows immediately from the
definition of covariance for indicator variables, which is easily seen
in the scalar form: $\cov(\mathbb{I}(X_A = i_A), \mathbb{I}(X_B = i_B)) = \Pr(X_A = i_A,
X_B = i_B) - \Pr(X_A = i_A)\Pr(X_B = i_B)$.
Eq.~\eqref{covariance} also covers the case when $A \cap B$ is
nonempty. In particular if $A = B = \{u\}$, then we recover
$\cov(\b{I}_u, \b{I}_u) = \diag(\mub_u) - \mub_u \mub_u^T$, which is the
covariance matrix for the marginal multinomial distribution of $\b{I}_u$.

From the preceding arguments, it becomes clear that the covariance
matrix for the full vector of indicator variables has a simple block
structure. 
Define $\b{I} = (\b{I}_A: A \in \C \cup \S)$ to be the vector
concatention of all the clique and separator indicator variables, and
let $\mub = (\mub_A: A \in \C \cup \S) = \E[\b{I}]$ be the
corresponding vector concatenation of marginals. Then it follows from
\eqref{covariance} that the covariance matrix is
\begin{equation}
\label{covariance2}
\Sigma := \cov(\b{I}, \b{I}) = \hat{\Sigma} - \mub \mub^T,
\end{equation}
where $\hat{\Sigma}$ is the matrix whose $(A, B)$ block is the marginal
distribution $\< \mub_{A,B} \>$. 
In the CGM model, the count vector $\n$ can be written as $\n =
\sum_{m=1}^N \b{I}^m$, where $\b{I}^1, \ldots, \b{I}^N$ are i.i.d. copies of
$\b{I}$. As a result, the moments of the CGM are obtained by scaling
the moments of $\b{I}$ by $N$. We thus arrive at the natural
moment-matching Gaussian approximation of the CGM.
\begin{definition}
The Gaussian CGM, denoted $\GCGM(N, \thetab)$ is the multivariate
normal distribution $\mathcal{N}(N\mub, N\Sigma)$, where $\mub$ is the
vector of all clique and separator marginals of the graphical model
with parameters $\thetab$, and $\Sigma$ is defined
in Equation~\eqref{covariance2}.
\end{definition}
In the following theorem, we show the GCGM is asymptotically correct
and it is a Gaussian graphical model, which will lead to efficient
inference algorithms.
%
\begin{theorem}
\label{main-theorem}
Let $\n^N \sim \CGM(N, \thetab)$ for $N =1,2,\ldots$. Then following are true:
\begin{itemize}
\item[(i)] The GCGM is asymptotically correct. That is, as $N
  \rightarrow \infty$ we have
\begin{equation}
\frac{1}{\sqrt{N}}(\n^N - N \mub) \law \mathcal{N}(\b{0}, \Sigma).
\end{equation}
\item[(ii)] The GCGM is a Gaussian graphical model with the same
  conditional independence structure as the CGM. 
  Let $\z \sim \GCGM(N, \thetab)$
  and let $A, B \in \C \cup \S$ be two sets that are
  separated by separator $S$ in $\mathcal{T}$. Then $\z_{A} \indep
  \z_{B} \mid \z_S$.
\end{itemize}
\end{theorem}
\begin{proof}
Part (i) is a direct application of the multivariate central
limit theorem to the random vector $\n^N$, which, as noted above, is a
sum of i.i.d. random vectors $\b{I}^1, \ldots, \b{I}^N$ with mean
$\mub$ and covariance $\Sigma$ \cite{feller1968}.

Part (ii) is a consequence of the fact that these conditional
independence properties hold for each $\n^N$
(Proposition~\ref{cgm-independence}), so they also hold in the limit as
$N \rightarrow \infty$. While this is intuitively clear, it seems
to require further justification, which is provided in the
supplementary material. 
\end{proof}

\subsection{Conditional Distributions}
The goal is to use inference in the GCGM as a tractable approximate
alternative inference method for CGMs. However, it is very difficult to 
compute the covariance matrix $\Sigma$ over all cliques. In particular, note that the
$(C,C')$ block requires the joint marginal $\<\mub_{C,C'}\>$, and if
$C$ and $C'$ are not adjacent in $\mathcal{T}$ this is hard to
compute. Fortunately, we can sidestep the problem completely by
leveraging the graph structure from Part~(ii) of
Theorem~\ref{main-theorem} to write the distribution as a product of
conditional distributions whose parameters are easy to compute (this
effectively means working with the inverse covariance matrix instead
of $\Sigma$). We then perform inference by Gaussian message passing
on the resulting model.

A challenge is that $\Sigma$ is not full rank, so the GCGM
distribution as written is degenerate and does not have a density.
This can be seen by noting that any vector $\n \sim \CGM(N; \thetab)$
with nonzero probability satisfies the affine consistency constraints
from Eq.~\eqref{cgm-base-measure}---for example, each vector $\n_C$
and $\n_S$ sums to the population size $N$---and that these affine
constraints also hold with probability one in the limiting
distribution. To fix this, we instead use a linear transformation
$\mathbb{T}$ to map $\z$ to a reduced vector $\tz = \mathbb{T} ~ \z$
such that the reduced covariance matrix $\tilde{\Sigma} = \mathbb{T} ~
\Sigma ~ \mathbb{T}^T$ is invertible. The work by
\citet{loh2013structure} proposed a minimal representation of the
graphical model in \eqref{model1}, and the corresponding random
variable has a full rank covariance matrix. We will find a
transformation $\mathbb{T}$ to project our indicator variable
$\mathbf{I}$ into that form. Then $\mathbb{T} \; \mathbf{I}$ (as well
as $\mathbb{T} \; \n$ and $\mathbb{T} \; \z$) will have a full rank
covariance matrix.

Denote by $\C^+$ the maximal and non-maximal cliques in the
triangulated graph.  Note that each $D \in \C^+$ must be a subset of
some $A \in \C \cup \S$ and each subset of $A$ is also a clique in
$\C^+$. For every $D\in \C^+$, let $\X_0^D =(\X \backslash \{L\})^{|D|}$
denote the space of possible configurations of $D$ after excluding the
largest value, $L$, from the domain of each variable in $D$.  The
corresponding random variable $\I$ in the minimal representation is
defined as \cite{loh2013structure}:
\begin{eqnarray}
\tI = (\I(\x_D = i_D): i_D \in \X_0^D, D \in \C^+ ) ~~. 
\end{eqnarray} 
$\tI_D$ can be calculated linearly from $\mathbf{I}_A$ when
$D\subseteq A$ via the matrix $\mathbb{T}_{D, A}$ whose $(i_D,i_A)$
entry is defined as
\begin{eqnarray} 
\mathbb{T}_{D, A}(i_D, i_A) &=& \I(i_D \sim_D i_A), \label{trans}
\end{eqnarray}
where $\sim_D$ means that $i_D$ and $i_A$ agree on the setting of the
variables in $D$.  It follows that $\tI_D = \mathbb{T}_{D, A} \;
\mathbf{I}_A$.  The whole transformation $\mathbb{T}$ can be built in
blocks as follows: For every $D \in \C^+$, choose $A \in \C \cup \S$
and construct the $\mathbb{T}_{D, A}$ block via (\ref{trans}).  Set
all other blocks to zero. Due to the redundancy of $\mathbf{I}$, there
might be many ways of choosing $A$ for $D$ and any one will work as
long as $D \subseteq A$.

\begin{proposition}
\label{prop-tz} Define $\mathbb{T}$ as above, and define 
$\tz = \mathbb{T} ~ \z$, $\tz_{A^+} = (\tz_D: D\subseteq A), A\in \C \cup \S$. 
Then 
\begin{itemize}
\item[(i)] If $A, B \in \C \cup S$ are
separated by $S$ in $\mathcal{T}$, it holds that $\tz_{A^+} \indep \tz_{B^+}
\mid \tz_{S^+}$.
\item[(ii)] The covariance matrix of $\tz$ has full rank. 
\end{itemize}
\end{proposition}
\begin{proof}
In the appendix, we show that for any $A \in \C \cup \S$,
$\mathbf{I}_A$ can be linearly recovered from $\tI_{A^+} = (\tI_{D}: D
\subseteq A)$.  So there is a linear bijection between $\mathbf{I}_A$
and $\tI_{A^+}$ (The mapping from $\mathbf{I}_A$ to $\tI_{A^+}$ has
been shown in the definition of $\mathbb{T}$).  The same linear bijection relation also exists
between $\n_A$ and $\tilden_{A^+} = \sum_{m=1}^{N} \tI_{A^+}^{m}$ and
between $\z_A$ and $\tz_{A^+}$.

Proof of (i): Since $\z_A \indep \z_B \mid \z_S$, it follows that
$\tz_{A^+} \indep \tz_{B^+} \mid \z_{S}$ because $\tz_{A^+}$ and
$\tz_{B^+}$ are deterministic functions of $\z_A$ and $\z_B$
respectively. Since $\z_S$ is a deterministic function of $\tz_{S^+}$,
the same property holds when we condition on $\tz_{S^+}$ instead of
$\z_S$.

Proof of (ii): The bijection between $\mathbf{I}$ and $\tI$ indicates
that the model representation of \citet{loh2013structure} defines the
same model as \eqref{model1}.  By \citet{loh2013structure}, $\tI$ has
full rank covariance matrix and so do $\tilden$ and $\tz$.
\end{proof}

With this result, the GCGM can be decomposed into conditional
distributions, and each distribution is a non-degenerate Gaussian
distribution.

Now let us consider the observations $\y=\{\y_D, D \in \D \}$, where
$\D$ is the set of cliques for which we have observations.  We require
each $D \in \D$ be subset of some clique $C\in \C$.  When choosing a
distribution $p(\y_D | \z_C)$, a modeler has substantial flexibility.
For example, $p(\y_D | \z_C)$ can be noiseless, $\y_D(i_D) =
\sum_{i_{C \backslash D}} \z_C (i_D, i_{C\backslash D})$, which
permits closed-form inference. Or $p(\y_D | \z_C)$ can consist of
independent noisy observations: $p(\y_D | \z_C) =
\prod_{i_D}p(\y_D(i_D) |\sum_{i_{C \backslash D}} \z_C (i_D,
i_{C\backslash D}))$.  With a little work, $p(\y_D | \z_C)$ can be
represented by $p(\y_D | \tz_{C^+})$. 

\subsection{Explicit Factored Density for Trees} \label{gcgm_tree} 

We describe how to decompose GCGM for the special case when the
original graphical model $G$ is a tree. We assume that only counts of
single nodes are observed. In this case, we can marginalize out edge
(clique) counts $\z_{\{u, v\}}$ and retain only node (separator) counts
$\z_u$.  Because the GCGM has a normal distribution, marginalization is
easy.  The conditional distribution is then defined only on node counts.
With the definition of $\tz$ in Proposition~\eqref{prop-tz} and the
property of conditional independence, we can write
\begin{equation}
\label{GCGM-factor}
p(\tz_1, \ldots, \tz_n) = p(\tz_r) \prod_{(u,v) \in E} p(\tz_v \mid \tz_u).
\end{equation}
Here $r \in V$ is an arbitrarily-chosen root node, and $E$ is the set
of \emph{directed} edges of $G$ oriented away from $r$.  The
marginalization of the edges greatly reduces the size of the inference
problem, and a similar technique is also applicable to general GCGMs.

Now specify the parameters of the Gaussian conditional densities
$p(\tz_v \mid \tz_u)$ in Eq.~\eqref{GCGM-factor}. Assume the 
blocks $\mathbb{T}_{u, u}$ and $\mathbb{T}_{v, v}$ are defined as \eqref{trans}. 
Let $\tilde{\mub}_u = \mathbb{T}_{u, u} \; \mub_{u}$ be the
marginal vector of node $u$ without its last entry, and let
$\<\tilde{\mub}_{u,v}\> = \mathbb{T}_{u, u} ~ 
\<\mub_{u,v}\> ~ \mathbb{T}_{v, v}^T$
be the marginal matrix over edge $(u,v)$, minus the final row and
column. Then the mean and covariance martix of the joint distribution are
\begin{eqnarray} 
\boldsymbol{\eta} := N \begin{bmatrix}\tilde{\mub}_u \\
  \tilde{\mub}_v \end{bmatrix},
\quad
N^2 \begin{bmatrix}
\diag(\tilde{\mub}_u)   &  \<\tilde{\mub}_{u,v}\> \\
\<\tilde{\mub}_{v,u}\>  &   \diag(\tilde{\mub}_v)
\end{bmatrix}
- \boldsymbol{\eta} \boldsymbol{\eta}^T. \label{uvjoint}
\end{eqnarray}
The conditional density $p(\tz_v \mid \tz_u)$ is obtained by standard
Gaussian conditioning formulas.

If we need to infer $\z_{\{u,v\}}$ from some distribution $q(\tz_{u}, \tz_{v})$, 
we first calculate the distribution $p(\tz_{\{u,v\}} | \tz_{u}, \tz_{v})$. 
This time we assume blocks $\mathbb{T}_{\{u, v\}^{+}, \{u, v\}} = 
(\mathbb{T}_{u, \{u, v\}} : D \in \{u, v\})$ are defined as \eqref{trans}. 
We can find the mean and variance of 
$p(\tz_u,\tz_v, \tz_{\{u, v\}})$ by applying linear transformation 
$\mathbb{T}_{\{u, v\}^{+}, \{u, v\}}$ on the mean and variance of $\z_{\{u, v\}}$. 
Standard Gaussian conditioning formulas then give the conditional 
distribution $p(\tz_{\{u,v\}} \mid \tz_u, \tz_v)$. Then we can recover 
the distribution of $\z_{\{u,v\}}$ from distribution 
$p(\tz_{\{u,v\}} | \tz_{u}, \tz_{v})q(\tz_{u}, \tz_{v})$.

{\bf Remark:} Our reasoning gives a completely different way
of deriving some of the results of 
\citet{loh2013structure} concerning the sparsity pattern of the inverse
covariance matrix of the sufficient statistics of a discrete graphical model.
The conditional independence in Proposition~\ref{cgm-independence} 
for the factored GCGM density translates
directly to the sparsity pattern in the precision matrix $\Gamma =
\tilde{\Sigma}^{-1}$. Unlike the reasoning of \citeauthor{loh2013structure}, we
derive the sparsity directly from the conditional independence
properties of the asymptotic distribution (which are inherited from
the CGM distribution) and the fact that the CGM and GCGM share the same
covariance matrix.

\section{Inference with Noisy Observations}

We now consider the problem of inference in the GCGM when the
observations are noisy.  Throughout the remainder of the paper, we
assume that the individual model---and, hence, the CGM---is a tree. In
this case, the cliques correspond to edges and the separators
correspond to nodes. We will also assume that only the nodes are
observed. For notational simplicity, we will assume that every node is
observed (with noise). (It is easy to marginalize out unobserved nodes
if any.) From now on, we use $uv$ instead of $\{u, v\}$ to represent 
edge clique. Finally, we assume that the entries have been dropped from
the vector $\z$ as described in the previous section so that it has
the factored density described in Eq.~\ref{GCGM-factor}. 

Denote the observation variable for node $u$ by $\mathbf{y}_u$, and
assume that it has a Poisson distribution.  In the (exact) CGM, this
would be written as $\mathbf{y}_u \sim
\mathrm{Poisson}(\mathbf{n}_{u})$. However, in our GCGM, this instead
has the form
\begin{eqnarray}
\mathbf{y}_{u} \sim  \mathrm{Poisson}(\lambda \barn_{u}),
\end{eqnarray}
where $\barn_{u}$ is the corresponding continuous variable and
$\lambda$ determines the amount of noise in the distribution.  Denote
the vector of all observations by $\mathbf{y}$.  Note that the missing
entry of $\z_u$ must be reconstructed from the
remaining entries when computing the likelihood.

With Poisson observations, there is no longer a closed-form solution
to message passing in the GCGM.  We address this by applying
Expectation Propagation (EP) with the Laplace approximation. This
method has been previously applied to nonlinear dynamical systems by
\citet*{ypma2005}.

\subsection{Inferring Node Counts}

In the GCGM with observations, the potential on each edge $(u, v) \in E$ is defined as 
\begin{eqnarray}
\psi(\barn_u, \barn_v) = \hspace{5.5cm} \nonumber\\ \hspace{0.8cm}
\left\{ 
\begin{array}{ll}
p(\barn_v,  \barn_u) p(\mathbf{y}_v | \barn_v) p(\mathbf{y}_u | \barn_u)
& \hspace{-0.1cm} \mbox{if } u \mbox{ is root} \\
p(\barn_v | \barn_u) p(\mathbf{y}_v | \barn_v) & \mbox{otherwise.}
\end{array}
\right.  \label{eppoten}
\end{eqnarray}
We omit the subscripts on $\psi$ for notational simplicity.  The joint
distribution of $(\barn_v, \barn_u)$ has mean and
covariance shown in \eqref{uvjoint}.

With EP, the model approximates potential on edge $(u, v) \in E$ with normal 
distribution in context $q_{\backslash uv}(\barn_u)$ and $q_{\backslash uv}(\barn_v)$. 
The context for edge $(u, v)$ is defined as
\begin{eqnarray}
q_{\backslash uv}(\barn_u) &=& \prod_{(u, v') \in E, v' \neq v} q_{uv'}(\barn_u)
\label{messageu}\\
q_{\backslash uv}(\barn_v) &=& \prod_{(u', v) \in E, u' \neq u} q_{u'v}(\barn_v), 
\label{messagev}
\end{eqnarray}
where each $q_{uv'}(\barn_u)$ and $q_{u'v}(\barn_v)$ have the form of normal densities. 

Let $\xi(\barn_u, \barn_v) = q_{\backslash uv}(\barn_u) q_{\backslash uv}(\barn_v)\psi(\barn_u, \barn_v)$.
The EP update of $q_{uv}(\barn_u)$ and $q_{uv}(\barn_v)$ is computed as
\begin{eqnarray}
q_{uv}(\barn_u) &=& \frac{\mathrm{proj}_{\barn_u} [\xi(\barn_u, \barn_v)]}{q_{\backslash uv}(\barn_u)}
\label{proju} \\
q_{uv}(\barn_v) &=& \frac{\mathrm{proj}_{\barn_v} [\xi(\barn_u, \barn_v)]}{q_{\backslash uv}(\barn_v)}. 
\label{projv}
\end{eqnarray}

The projection operator, $\mathrm{proj}$, is computed in two
steps. First, we find a joint approximating normal distribution via
the Laplace approximation and then we project this onto each of the
random variables $\barn_u$ and $\barn_v$.  In the Laplace
approximation step, we need to find the mode of $\log \xi(\barn_u,
\barn_v)$ and calculate its Hessian at the mode to obtain the mean and
variance of the approximating normal distribution:
\begin{eqnarray} 
\mu^\xi_{uv} &=& \operatorname*{arg\,max}_{(\barn_u, \barn_v)} \log \xi(\barn_u, \barn_v) \label{laplacemean} \\
\Sigma^\xi_{uv} &=& \left( \operatorname*{\nabla^2}_{(\barn_u, \barn_v) 
= \mu^\xi_{uv}} \log \xi(\barn_u, \barn_v) \right)^{-1}.
\label{laplacevar}
\end{eqnarray}

The optimization problem in (\ref{laplacemean}) is solved by
optimizing first over $\barn_u$ then over $\barn_v$. The
optimal value of $\barn_u$ can be computed in closed form
in terms of $\barn_v$, since only normal densities are involved.  
Then the optimal value of $\barn_v$ is found via gradient methods (e.g., BFGS).
The function $\log \xi(\barn_u, \barn_v)$ is concave, so we can
always find the global optimum. Note that this
decomposition approach only depends on the tree structure of the model
and hence will work for any observation distribution.

At the mode, we find the mean and variance of the normal distribution
approximating $p(\barn_u, \barn_v | \mathbf{y})$ via
\eqref{laplacemean} and \eqref{laplacevar}.  With this distribution,
the edge counts can be inferred with the method of Section
\ref{gcgm_tree}. In the projection step in \eqref{proju} and
\eqref{projv}, this distribution is projected to one of $\barn_u$ or
$\barn_v$ by marginalizing out the other.

\subsection{Complexity analysis}

What is the computational complexity of inference with the GCGM?  When
inferring node counts, we must solve the optimization problem and
compute a fixed number of matrix inverses. Each matrix inverse takes
time $L^{2.5}$.  In the Laplace approximation step, each gradient
calculation takes $O(L^2)$ time.  Suppose $m$ iterations are needed.
In the outer loop, suppose we must perform $r$ passes of EP message
passing and each iteration sweeps through the whole tree. Then the
overall time is $O(r|E| \max (m L^2, L^{2.5}))$.  The maximization
problem in the Laplace approximation is smooth and concave, so it is
relatively easy. In our experiments, EP usually converges within 10
iterations.

In the task of inferring edge counts, we only consider the complexity
of calculating the mean, as this is all that is needed in our
applications. This part is solved in closed form, with the most
time-consuming operation being the matrix inversion. By exploiting the
simple structure of the covariance matrix of $\barn_{uv}$ , we can
obtain an inference method with time complexity of $O(L^3)$.

\section{Experimental Evaluation}

In this section, we evaluate the performance of our method and compare
it to the MAP approximation of \citet*{Sheldon2013}.  The evaluation
data are generated from the bird migration model introduced in
\citet{Sheldon2013}.  This model simulates the migration of a
population of $M$ birds on an $L = \ell \times \ell$ map.  The entire
population is initially located in the bottom left corner of the
map. Each bird then makes independent migration decisions for $T=20$
time steps. The transition probability from cell $i$ to cell $j$ at
each time step is determined by a logistic regression equation that
employs four features. These features encode the distance from cell
$i$ to cell $j$, the degree to which cell $j$ falls near the path from
cell $i$ to the destination cell in the upper right corner, the degree
to which cell $j$ lies in the direction toward which the wind is
blowing, and a factor that encourages the bird to stay in cell $i$.
Let $\mathbf{w}$ denote the parameter vector for this logistic
regression formula.  In this simulation, the individual model for each
bird is a $T$-step Markov chain $X=(X_1, \ldots, X_{20})$ where the
domain of each $X_t$ consists of the $L$ cells in the map.  The CGM
variables $\mathbf{n} = (\mathbf{n}_1, \mathbf{n}_{1,2}, \mathbf{n}_2,
\ldots, \mathbf{n}_T)$ are vectors of length $L$ containing counts of
the number of birds in each cell at time $t$ and the number of birds
moving from cell $i$ to cell $j$ from time $t$ to time $t+1$.  We will
refer to these as the ``node counts'' (N) and the ``edge counts''
(E). At each time step $t$, the data generation model generates an
observation vector $\mathbf{y}_t$ of length $L$ which contains noisy
counts of birds at all map cells at time $t$, $\mathbf{n}_t$. The
observed counts are generated by a Poisson distribution with unit
intensity.

We consider two inference tasks. In the first task, the parameters of
the model are given, and the task is to infer the expected value of
the posterior distribution over $\mathbf{n}_t$ for each time step $t$
given the observations $\mathbf{y}_1, \ldots, \mathbf{y}_T$ (aka
``smoothing'').  We measure the accuracy of the node counts and edge
counts separately.

An important experimental issue is that we cannot compute the true MAP
estimates for the node and edge counts.  Of course we have the values
generated during the simulation, but because of the noise introduced
into the observations, these are not necessarily the expected values
of the posterior. Instead, we estimate the expected values by running
the MCMC method \citep{Sheldon2011} for a burn-in period of 1 million Gibbs iterations and
then collecting samples from 10 million Gibbs iterations and averaging
the results.  We evaluate the accuracy of the approximate methods as the
relative error $||\mathbf{n}_{app} - \mathbf{n}_{mcmc}||_1 /
||\mathbf{n}_{mcmc}||_1$, where $\mathbf{n}_{app}$ is the approximate
estimate and $\mathbf{n}_{mcmc}$ is the value obtained from the Gibbs
sampler.  In each experiment, we report the mean and standard
deviation of the relative error computed from 10 runs. Each run
generates a new set of values for the node counts, edge counts, and
observation counts and requires a separate MCMC baseline run.

We compare our method to the approximate MAP method introduced by
\citet{Sheldon2013}.  By treating counts as continuous and
approximating the log factorial function, their MAP method finds the
approximate mode of the posterior distribution by solving a convex
optimization problem.  Their work shows that the MAP method is much
more efficient than the Gibbs sampler and produces inference results
and parameter estimates very similar to those obtained from long MCMC
runs.
 
The second inference task is to estimate the parameters $\mathbf{w}$
of the transition model from the observations.  This is performed via
Expectation Maximization, where our GCGM method is applied to compute
the E step. We compute the relative error with respect to the true
model parameters. 

Table~\ref{varpop} compares the inference accuracy of the approximate
MAP and GCGM methods. In this table, we fixed $L=36$, set the logistic
regression coefficient vector $\mathbf{w}=(1, 2, 2, 2)$, and varied
the population size $N \in \{36, 360, 1080, 3600\}$.  At the smallest
population size, the MAP approximation is slightly better, although
the result is not statistically significant.  This makes sense, since
the Gaussian approximation is weakest when the population size is
small. At all larger population sizes, the GCGM gives much more
accurate results. Note that the MAP approximation exhibits much higher
variance as well.

\begin{table}[ht]
\caption{Relative error in estimates of node counts (``N'') and edge
  counts (``E'') for different population sizes $N$.}
\label{varpop}
\centering
\scalebox{0.85}{
\begin{tabular}{ccccc}
  \hline
 $N =$ & 36 & 360 & 1080 & 3600 \\ 
  \hline
  MAP(N) & .173$\pm$.020 & .066$\pm$.015 & .064$\pm$.012 & .069$\pm$.013 \\ 
  MAP(E) & .350$\pm$.030 & .164$\pm$.030 & .166$\pm$.027 & .178$\pm$.025 \\ 
  \hline
  GCGM(N) & .184$\pm$.018 & .039$\pm$.007 & .017$\pm$.003 & .009$\pm$.002 \\ 
  GCGM(E) & .401$\pm$.026 & .076$\pm$.008 & .034$\pm$.003 & .017$\pm$.002 \\ 
   \hline
 \end{tabular}}
\end{table}

Our second inference experiment is to vary the magnitude of the
logistic regression coefficients. With large coefficients, the
transition probabilities become more extreme (closer to 0 and 1), and
the Gaussian approximation should not work as well.  We fixed $N=1080$
and $L=36$ and evaluated three different parameter vectors:
$\mathbf{w}_{0.5} = (0.5, 1, 1, 1)$, $\mathbf{w}_1 = (1, 2, 2, 2)$ and
$\mathbf{w}_2 = (2, 4, 4, 4)$. Table~\ref{varparam} shows that for
$\mathbf{w}_{0.5}$ and $\mathbf{w}_{1}$, the GCGM is much more
accurate, but for $\mathbf{w}_{2}$, the MAP approximation gives a
slightly better result, although it is not statistically significant
based on 10 trials. 

\begin{table}[ht]
\caption{Relative error in estimates of node counts (``N'') and edge
  counts (``E'') for different settings of the logistic regression
  parameter vector $\mathbf{w}$}
\label{varparam}
\centering
\begin{tabular}{cccc}
  \hline
 & $\mathbf{w}_{0.5}$ & $\mathbf{w}_{1}$ & $\mathbf{w}_2$ \\ 
 \hline
  MAP(N) & .107$\pm$.014 & .064$\pm$.012 & .018$\pm$.004 \\ 
  MAP(E) & .293$\pm$.038 & .166$\pm$.027 & .031$\pm$.004 \\ 
  \hline
  GCGM(N) & .013$\pm$.002 & .017$\pm$.003 & .024$\pm$.004 \\ 
  GCGM(E) & .032$\pm$.004 & .034$\pm$.003 & .037$\pm$.005 \\
\hline
\end{tabular}
\end{table}

\begin{figure*}
\centering
\begin{tabular}{ccc}
\includegraphics[width=0.32\textwidth]{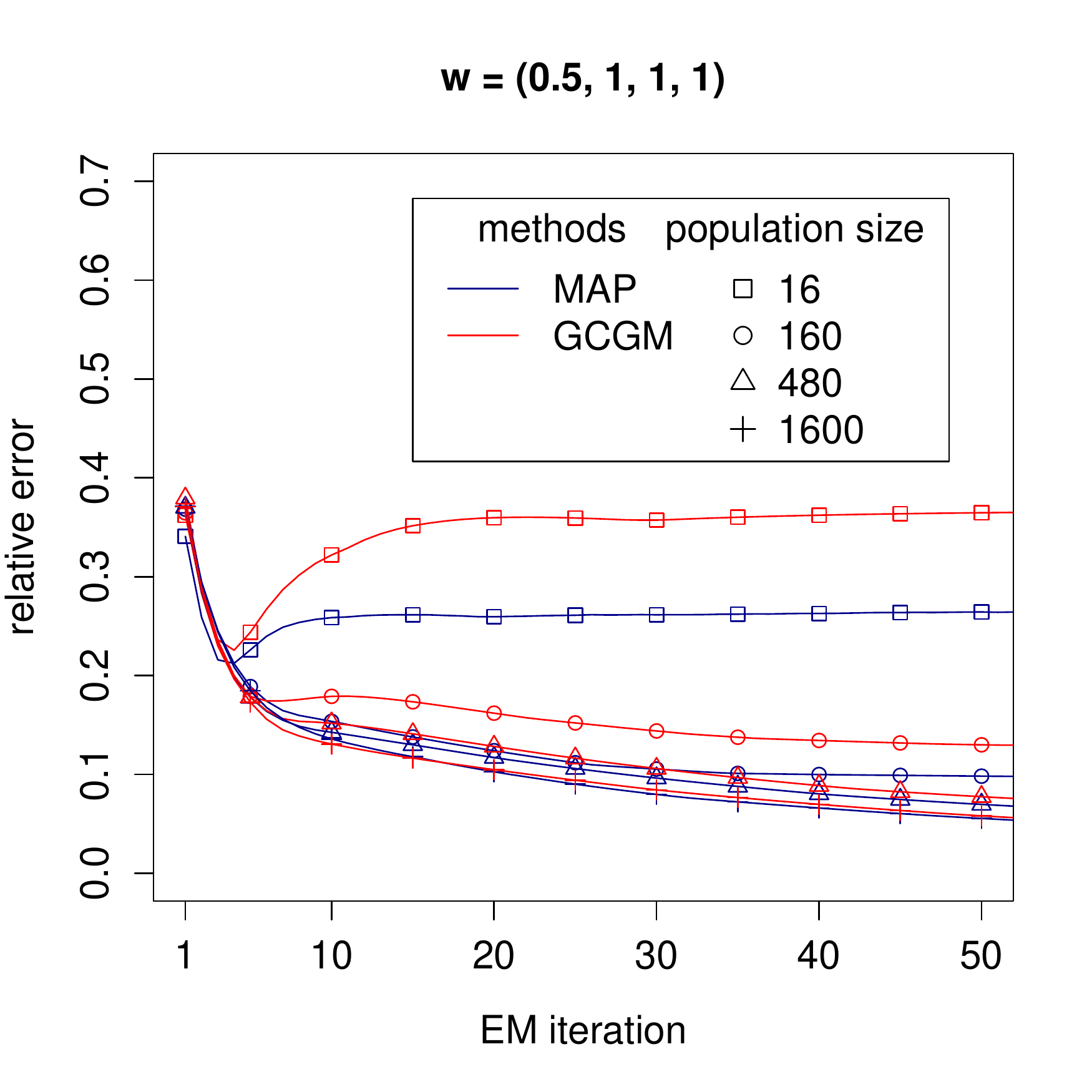}
\includegraphics[width=0.32\textwidth]{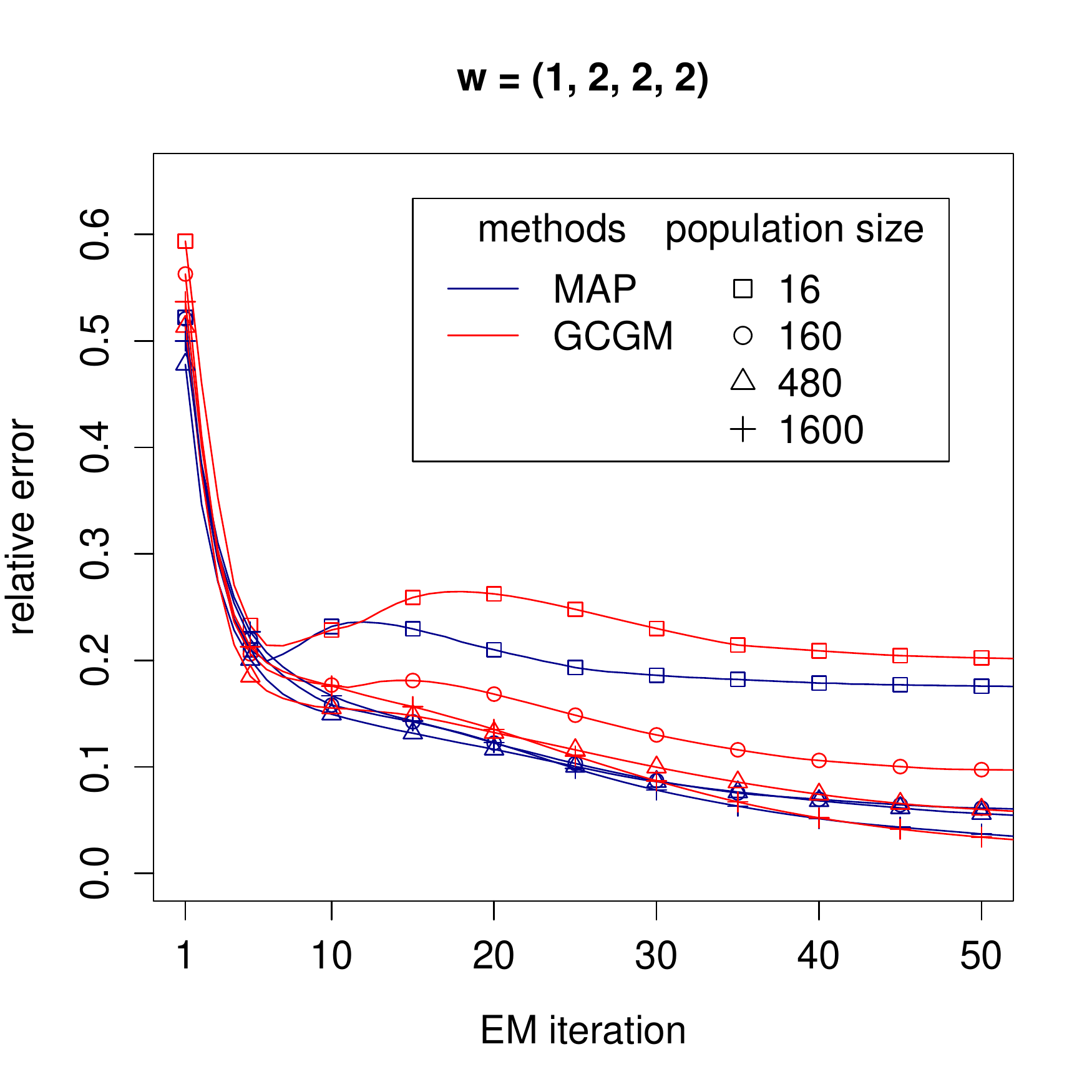}
\includegraphics[width=0.32\textwidth]{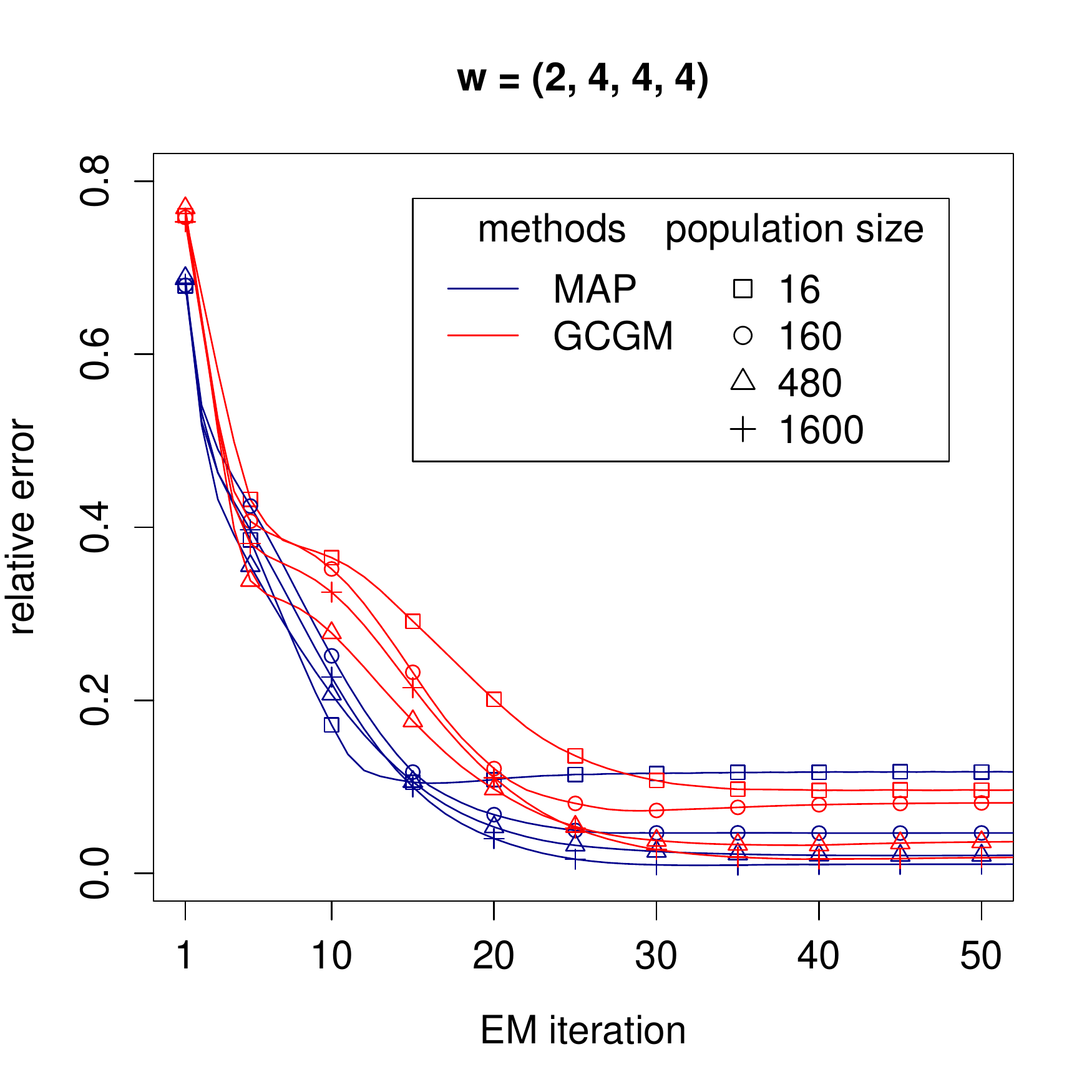}
\end{tabular}
\caption{EM convergence curve different feature coefficient and population sizes}
\label{learncurve}
\end{figure*}

Our third inference experiment explores the effect of varying the size
of the map. This increases the size of the domain for each of the
random variables and also increases the number of values that must be
estimated (as well as the amount of evidence that is observed).  We
vary $L\in \{16, 25, 36, 49\}$. We scale the population size
accordingly, by setting $N = 30 L$.  We use the coefficient vector
$\mathbf{w}_1$.  The results in Table~\ref{varcell} show that for the
smallest map, both methods give similar results. But as the number of
cells grows, the relative error of the MAP approximation grows rapidly
as does the variance of the result. In comparison, the relative error
of the GCGM method barely changes.

\begin{table}[ht]
\caption{Relative inference error with different map size}
\label{varcell}
\centering
\scalebox{0.85}{
\begin{tabular}{ccccc}
  \hline
 $L = $ & 16 & 25 & 36 & 49 \\ 
  \hline
  MAP(N) & .011$\pm$.005 & .025$\pm$.007 & .064$\pm$.012 & .113$\pm$.015 \\ 
  MAP(E) & .013$\pm$.004 & .056$\pm$.012 & .166$\pm$.027 & .297$\pm$.035 \\ 
  \hline
  GCGM(N) & .017$\pm$.003 & .017$\pm$.003 & .017$\pm$.003 & .020$\pm$.003 \\ 
  GCGM(E) & .024$\pm$.002 & .027$\pm$.003 & .034$\pm$.003 & .048$\pm$.005 \\ 
\hline
\end{tabular}}
\end{table}

We now turn to measuring the relative accuracy of the methods during
learning.  In this experiment, we set $L=16$ and vary the population
size for $N \in \{16, 160, 480, 1600\}$.  After each EM iteration, we
compute the relative error as $||\mathbf{w}_{learn} -
\mathbf{w}_{true}||_1 / ||\mathbf{w}_{true}||_1$, where
$\mathbf{w}_{learn}$ is the parameter vector estimated by the learning
methods and $\mathbf{w}_{true}$ is the parameter vector that was used
to generate the data.  Figure~\ref{learncurve} shows the training
curves for the three parameter vectors $\mathbf{w}_{0.5},
\mathbf{w}_1,$ and $\mathbf{w}_2$. The results are consistent with our
previous experiments.  For small population sizes ($N=16$ and
$N=160$), the GCGM does not do as well as the MAP approximation. In some
cases, it overfits the data.  For $N=16$, the MAP approximation also
exhibits overfitting.  For $\mathbf{w}_2$, which creates extreme
transition probabilities, we also observe that the MAP approximation
learns faster, although the GCGM eventually matches its performance
with enough EM iterations.

Our final experiment measures the CPU time required to perform
inference.  In this experiment, we varied $L \in \{16, 36, 64, 100,
144\}$ and set $N=100 L$. We used parameter vector $\mathbf{w}_1$.  We
measured the CPU time consumed to infer the node counts and the edge
counts. The MAP method infers the node and edge counts jointly,
whereas the GCGM first infers the node counts and then computes the
edge counts from them.  We report the time required for computing
just the node counts and also the total time required to compute the
node and edge counts.  Figure~\ref{learntime} shows that the running
time of the MAP approximation is much larger than the running time of
the GCGM approximation. For all values of $L$ except 16, the average
running time of GCGM is more than 6 times faster than for the MAP
approximation. The plot also reveals that the computation time of GCGM
is dominated by estimating the node counts. A detailed analysis of the
implementation indicates that the Laplace optimization step is the
most time-consuming.

\begin{figure}
\centering
\includegraphics[width=0.4\textwidth]{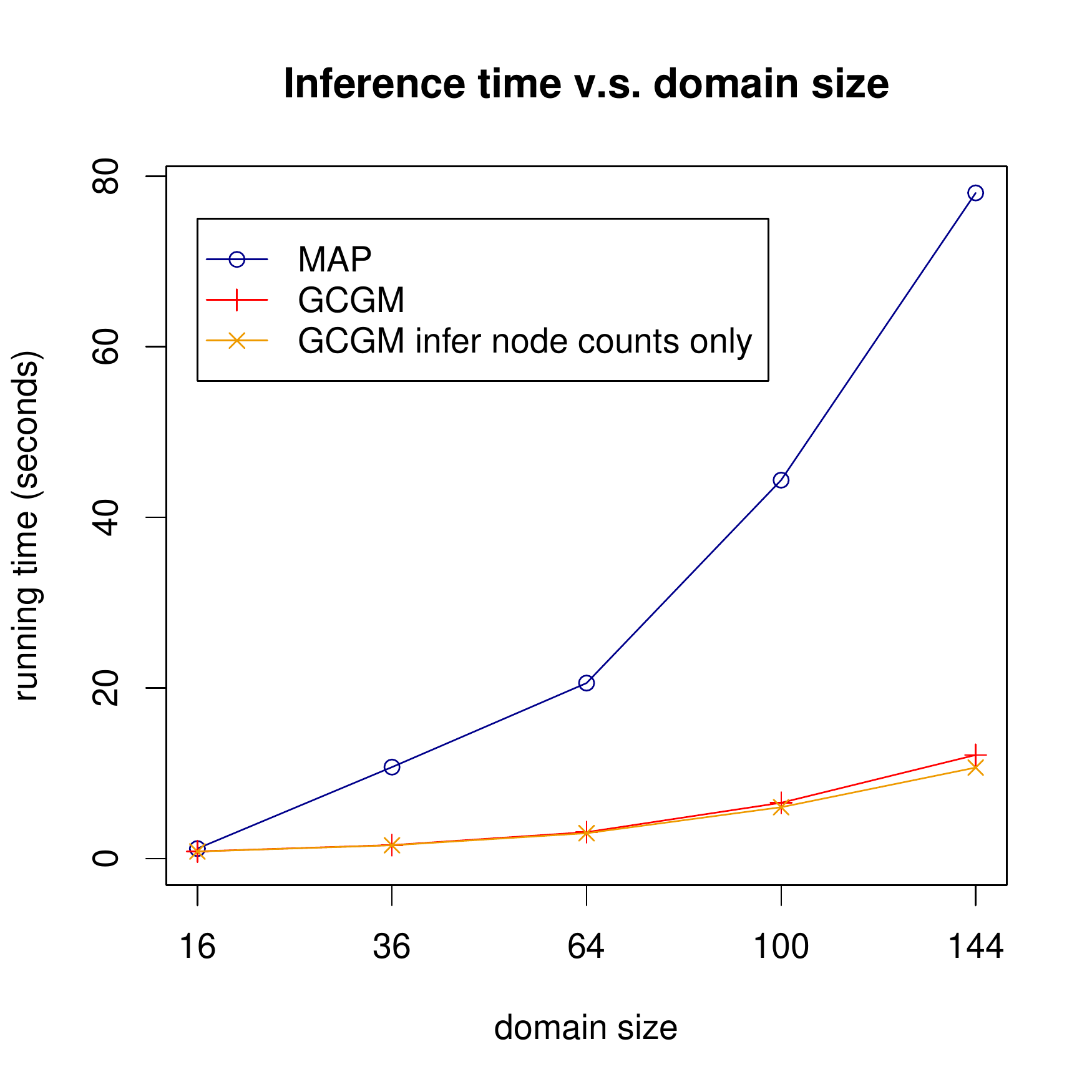}
\caption{A comparison of inference run time with different numbers of cells $L$}
\label{learntime}
\end{figure}

In summary, the GCGM method achieves relative error that matches or is
smaller than that achieved by the MAP approximation.  This is true
both when measured in terms of estimating the values of the latent
node and edge counts and when estimating the parameters of the
underlying graphical model.  The GCGM method does this while running
more than a factor of 6 faster. The GCGM approximation is particularly
good when the population size is large and when the transition
probabilities are not near 0 or 1.  Conversely, when the population
size is small or the probabilities are extreme, the MAP approximation
gives better answers although the differences were not statistically
significant based on only 10 trials.  A surprising finding is that the
MAP approximation has much larger variance in its answers than the
GCGM method.

\section{Concluding Remarks}

This paper has introduced the Gaussian approximation (GCGM) to the
Collective Graphical Model (CGM).  We have shown that for the case
where the observations only depend on the separators, the GCGM is the
limiting distribution of the CGM as the population size $N\rightarrow
\infty$.  We showed that the GCGM covariance matrix maintains the
conditional independence structure of the CGM, and we presented a
method for efficiently inverting this covariance matrix.  By applying
expectation propagation, we developed an efficient algorithm for
message passing in the GCGM with non-Gaussian
observations. Experiments on a bird migration simulation showed that
the GCGM method is at least as accurate as the MAP approximation of
\citet{Sheldon2013}, that it exhibits much lower variance, and that it
is 6 times faster to compute.

\subsection*{Acknowledgement}
This material is based upon work supported by the National
Science Foundation under Grant No. 1125228.

\bibliography{gcgm} 
\bibliographystyle{icml2014}

\clearpage
\appendix
\section{Proof of Proposition~\ref{cgm-distribution}}

The usual way of writing the CGM distribution is to replace
$f(\n;\thetab)$ in Eq.~\eqref{cgm-prob} by
\begin{equation}
\label{fprime}
f'(\n; \thetab) =  
\frac{\prod_{C \in \C, i_C \in \X^{|C|}} \mu_{C}(i_C)^{\n_{C}(i_C)}}
{\prod_{S \in \S, i_S \in \X^{|S|}} \big(\mu_{S}(i_S)^{\n_{S}(i_S)}\big)^{\nu(S)}}
\end{equation}
We will show that $f(\n; \thetab) = f'(\n;\thetab)$ for any $\n$ such
that $h(\n) > 0$ by showing that both descibe the probability of an
ordered sample with sufficient statistics $\n$. Indeed, suppose there
exists some ordered sample $\b{X}
= (\x^{1}, \ldots, \x^{N})$ with sufficient
statistics $\n$. Then it is clear from inspection of
Eq.~\eqref{cgm-prob} and Eq.~\eqref{fprime} that $f(\n; \thetab)
= \prod_{m=1}^N p(\x^{m}; \thetab) = f'(\n; \thetab) $ by the junction
tree reparameterization of $p(\x; \thetab)$
\cite{wainwright2008graphical}. It only remains to show that such an
$\b{X}$ exists whenever $h(\n) > 0$. This is exactly what was shown by
\citet{Sheldon2011}: for junction trees, the hard
constraints of Eq.~\eqref{cgm-base-measure}, which enforce local
consistency on the integer count variables, are equivalent to the
global consistency property that there exists some ordered sample
$\b{X}$ with sufficient statistics equal to $\n$. (Since these are
integer count variables, the proof is quite different from the similar
theorem that local consistency implies global consistency for marginal distributions.)
We briefly note two interesting corollaries to this argument. First,
by the same reasoning, \emph{any} reparameterization of $p(\x;
\thetab)$ that factors in the same way can be used to replace $f(\n;
\thetab)$ in the CGM distribution.
Second, we can see that the base measure $h(\n)$ is exactly the
\emph{number of different ordered samples} with sufficient statistics
equal to $\n$. 

\section{Proof of Theorem~\ref{main-theorem}: Additional Details}
\newcommand{\etab}{\boldsymbol{\eta}}
Suppose $\{\n^N\}$ is a sequence of
random vectors that converge in distribution to $\n$, and
that $\n^N_A$, $\n^N_B$, and $\n^N_S$ are subvectors that
satisfy 
\begin{equation}
\label{ci}
\n^N_A \indep \n^N_B \mid \n^N_S
\end{equation}
for all $N$.
Let $\alpha$, $\beta$, and $\gamma$ be measurable sets in the appropriate
spaces and define
\begin{align}
\label{z}
z &= \Pr(\n_A \in \alpha, \n_B \in \beta \mid \n_S \in \gamma) - \\
& \hphantom{{} = 1} \Pr(\n_A \in \alpha \mid \n_S \in \gamma) \Pr(\n_B \in \beta \mid \n_S\in \gamma) \notag 
\end{align}
Also let $z^N$ be the same expression but with all instances of $\n$
replaced by $\n^N$ and note that $z^N = 0$ for all $N$ by the assumed
conditional independence property of Eq.~\eqref{ci}. Because the
sequence $\{\n^N\}$ converges in distribution to $\n$, we have
convergence of each term in $z^N$ to the corresponding term in $z$,
which means that
\[
z = \lim_{N \rightarrow \infty} z^N = \lim_{N \rightarrow \infty} 0 = 0,
\]
so the conditional independence property of Eq.~\eqref{ci} also holds
in the limit.

\section{Proof of Theorem~\ref{prop-tz}: Linear Function from \texorpdfstring{$\tI$}{tilde I} to 
\texorpdfstring{$\mathbf{I}$}{I}}

We need to show $\mathbf{I}_A$ can be recovered from $\tI_{A^+}$ 
with a linear function. 

Suppose the last indicator variable in $\mathbf{I}_A$ is $i^{0}_A$, which corresponds to the setting
that all nodes in $A$ take value $L$. Let $\mathbf{I}'_A$ be a set of indicators which contains
all entries in $\mathbf{I}_A$ but the last one $i^{0}_A$. Then $\mathbf{I}_A$ can be recovered 
from $\mathbf{I}'_A$ by the constraint $\sum_{i_A} \mathbf{I}_A(i_A) = 1$.  

Now we only need to show that $\mathbf{I}'_A$ can be recovered from $\mathbf{I}_{A^+}$ linearly. 
We claim that there exists an invertible matrix $\mathbb{H}$ such that 
$\mathbb{H} ~ \mathbf{I}'_A =  \tI_{A^+}$. 

Showing the existence of $\mathbb{H}$. Let $\tI_{A^+}(i_D)$ be the $i_D$ entry of $\tI_{A^+}$, 
which is for configuration $i_D$ of clique $D, D \subseteq A$.  
\begin{eqnarray}
\tI_{A^+}(i_D) &=& \sum_{i_{A \backslash D}} \mathbf{I}'_{A}(i_D, i_{A \backslash D})
\end{eqnarray}
Since no nodes in $D$ take value $L$ by definition of $\tI_D$, $(i_D, i_{A \backslash D})$  
{\em cannot} be the missing entry $i^0_A$ of  $\mathbf{I}'_{A}$, and the equation is always valid.  

Showing that $\mathbb{H}$ is square. For each $D$, there are $(L - 1)^{|D|}$ entries, and 
$A$ has $\binom{|A|}{|D|}$ sub-cliques with size $|D|$. So $\tI_{A^+}$ have overall 
$L^{|A|} - 1$ entries, which is the same as $\mathbf{I}'_{A}$. So $\mathbb{H}$ is a square matrix. 

We view $\mathbf{I}'_A$ and $\tI_{A^+}$ as matrices and each row is a indicator 
function of graph configurations. Since no trivial linear combination of $\tI_{A^+}$
is a constant by the conclusion in \citet*{loh2013structure}, $\tI_{A^+}$ has linearly 
independent columns. Therefore, $\mathbb{H}$ must have full rank and $\mathbf{I}'_{A}$  must have 
linearly independent columns.

\end{document}